\documentclass{amsart}
\usepackage[utf8]{inputenc}
\usepackage{amssymb}
\usepackage{amsfonts}
\usepackage{amsmath}
\usepackage{amsthm}
\usepackage{mathrsfs}
\usepackage{hyperref}
\usepackage{mathtools}
\usepackage{amsrefs}
\usepackage{bm}
\usepackage{piecewiseLinearMacros-private}
\usepackage{templatedShortcuts-private}
\usepackage{restate-private}
\title[Piecewise linear functions representable with neural networks]{Piecewise linear functions representable with infinite width shallow ReLU neural networks}
\author{Sarah McCarty } 

\address{Department of Mathematics, Iowa State University,
Ames, Iowa 50011} 
\email{smmcc@iastate.edu}


\begin{document}
\begin{abstract}
    This paper analyzes representations of continuous piecewise linear functions with infinite width, finite cost shallow neural networks using the rectified linear unit (ReLU) as an activation function.
    Through its integral representation, a shallow neural network can be identified by the corresponding signed, finite measure on an appropriate parameter space.
    We map these measures on the parameter space to measures on the projective $n$-sphere cross $\R$, allowing points in the parameter space to be bijectively mapped to hyperplanes in the domain of the function.
    We prove a conjecture of Ongie et al.
    that every continuous piecewise linear function expressible with this kind of infinite width neural network is expressible as a finite width shallow ReLU neural network.
\end{abstract}
\maketitle
\section{Introduction}
We consider shallow neural networks which use rectified linear unit (ReLU) as the activation function.
It is well-known ReLU has universal approximation properties on compact domains, and in practice has advantages over sigmoidal activation functions \cites{cybenkoApproximationSuperpositionsSigmoidal1989,oostwalHiddenUnitSpecialization2021}.
Finite width shallow neural networks with $n+1$-dimensional input take the form
\begin{equation}
    f(\vecx{x}) = c_0+\sum_{i=1}^k c_i\sigma\left(\veca{a}_i\cdot \vecx{x} - b_i\right)
    \label{finiteNetwork}
\end{equation}
where $\veca{a}_i\in\mathbb{S}^n$ (the unit sphere in $\R^{n+1}$) and $b_i,c_i\in\R$ for all $i$.
\\

Generalizing to infinite width neural networks transforms the sum to an integral and the weights $c_i$ to a signed measure $\mu$ on $\mathbb{S}^n\times\R$ where
\begin{equation}
    f(\vecx{x}) = \int_{\mathbb{S}^n\times\R} \sigma\left(\veca{a}\cdot \vecx{x} - b\right)\dif\mu(\veca{a},b)+c_0.
    \label{usualGeneralForm}
\end{equation}
Some authors choose instead to have the measure and integral over all of $\R^{n+1}\times\R$.
An important class of functions are those representable with an infinite width neural network with finite representation cost, which corresponds with $|\mu|\left(\mathbb{S}^n\times\R\right)<\infty$ \cite{bartlettValidGeneralizationSize}.
Similar classes of functions are studied in \cite{ePrioriEstimatesPopulation2019} as Barron spaces and in \cite{bachBreakingCurseDimensionality2017} as $\mathcal{F}_1$.
\\

To ensure the integral in \autoref{usualGeneralForm} is well-defined, we can require $\mu$ has a finite first moment where $\int_{\mathbb{S}^n\times\R} |b|\dif|\mu|(\veca{a},b)<\infty$.
Alternatively, Ongie et al.
in \cite{ongieFunctionSpaceView2019} writes the integral in the form
\begin{equation}
    \int_{\mathbb{S}^n\times\R} \sigma(\veca{a}\cdot \vecx{x}-b)-\sigma(-b)\dif\mu(\veca{a},b)+c_0,
    \label{ongieGeneralForm}
\end{equation}
so the integral is well-defined whenever $\mu$ is a finite measure.
Since they differ only by a constant, the class of functions representable with a finite measure in the form of \autoref{usualGeneralForm} is a subclass of the class of functions representable with a finite measure in the form of \autoref{ongieGeneralForm} \cite{ongieFunctionSpaceView2019}.
Therefore, we choose to consider integral representations in the form of \autoref{ongieGeneralForm}.
\\

Since $\sigma(\veca{a}\cdot \vecx{x}-b)$ is a ridge function, integral representations are also naturally studied as the dual ridgelet transform on distributions \cites{sonodaNeuralNetworkUnbounded2017,candesHarmonicAnalysisNeural1999,kostadinovaRidgeletTransformDistributions2014}.
Functions with these representations are then often analyzed with the Radon transform \cites{sonodaNeuralNetworkUnbounded2017, kostadinovaRidgeletTransformDistributions2014,ongieFunctionSpaceView2019}.
Savarese in \cite{savareseHowInfiniteWidth2019} characterized which one-dimensional functions are representable with infinite width, finite cost shallow ReLU neural networks.
\\

Many finitely piecewise linear functions cannot be represented with a finite width ReLU shallow network, including all non-trivial compactly supported piecewise linear functions on $\R^n$, $n\geq 2$ \cite{ongieFunctionSpaceView2019}.
Lower and upper bounds on number of layers needed to represent continuous piecewise linear functions with finite width, deep neural networks have been established \cites{heReLUDeepNeural2020, aroraUnderstandingDeepNeural2018}.
However, the class of infinite width ReLU networks is certainly more expressive than the class of finite width networks in general, such as being able to express some non-piecewise linear functions \cite{savareseHowInfiniteWidth2019}.
It is not obvious if the class of infinite width shallow ReLU neural networks can express a finitely piecewise linear function that the class of finite width shallow ReLU networks cannot.
By decomposing measures, E and Wojtowytsch in \cite{e.RepresentationFormulasPointwise2022} established the set of points of non-differentiability of a function in a Barron space must be a subset of a countable union of affine subspaces.
However, proper subsets are possible.
In \cite{ongieFunctionSpaceView2019}, Ongie et al.
proved many compactly supported piecewise linear functions are not representable with finite cost, infinite width shallow ReLU neural networks.
This led to the following conjecture.
\begin{conjecture}[Ongie et al., \cite{ongieFunctionSpaceView2019}]
    A continuous piecewise linear function $f$ has finite representation cost if and only if it is exactly representable by a finite width shallow neural network.
\end{conjecture}
\noindent A finite representation cost corresponds with the existence of a finite measure $\mu$ such that $f$ admits a representation in the form of \autoref{ongieGeneralForm}.
Our main result is to prove the conjecture.
\begin{theorem*}
    \restateBody{representableByDeltaPointMassesN}
\end{theorem*}
\subsection{Notation}
For $m\in\N$, let $[m]\coloneqq \{1,\ldots,m\}$.
\\
The rectified linear unit (ReLU) function from $\R$ to $\R$ is denoted $\sigma(t)$ and defined as $\sigma(t)\coloneqq \max\{0,t\}$.\\
The pushforward measure of measure $\mu$ induced by a mapping $\varphi$ is denoted $\mu\circ \varphi^{-1}$.\\
Regular lower case Latin letter variables generally represent real numbers: $x,y,z,t\in\R$.
Once $n$ is fixed, bold lower case Latin letter variables indicate elements of $\R^{n+1}$: $\veca{a},\vecx{x}\in\R^{n+1}$.
Bold lower case Greek letter variables indicate elements of $\R^n$: $\vecz{z},\vecu{u}\in\R^n$.
\\
Let $\hyp{a}{b}\coloneqq \{ \vecx{x}\in\R^{n+1} \; | \; \veca{a}\cdot \vecx{x} = b \}$.\\
For $m\in\N$, $m-1$ dimensional affine subspaces in $\R^m$ are called hyperplanes.
The $m$-sphere in $\R^{m+1}$ is denoted $\mathbb{S}^m$.
Let $\bm{e}_{m+1}\coloneqq\lDir 0,\ldots, 0, 1\rDir\in\mathbb{S}^m$.
\\
Let $\mathcal{S}^0\coloneqq\{1\}\subseteq\mathbb{S}^0$.
For $m\geq 1$, let $\mathcal{S}^m$ be defined as
\[
    \mathcal{S}^m \coloneqq \{ \vecx{x}\in \mathbb{S}^m \; | \; \bm{e}_{m+1} \cdot \vecx{x} > 0\}\cup \{(x_1,\ldots,x_{m},0) \; | \; (x_1\ldots,x_m)\in\mathcal{S}^{m-1}\}\subseteq\mathbb{S}^m.
\]
Call $\mathcal{S}^m$ a half $m$-dimensional hypersphere.
Let $-\mathcal{S}^m$ denote the pointwise negation of all the points in $\mathcal{S}^m$.
By simple induction, exactly one of $\vecx{x},-\vecx{x}\in\mathcal{S}^m$ for all $\vecx{x}\in\mathbb{S}^m$.
Therefore, $\mathbb{S}^m=\mathcal{S}^m\sqcup\left( -\mathcal{S}^m \right)$.
\\
Let $D_{\vecd{d}^+} f(\vecx{x})$ denote the one-sided directional derivative of $f$ in the positive direction of $\vecd{d}$ for $\vecd{d}\in\mathbb{S}^n$.\\
For any metric space $W$, $\mathcal{B}(W)$ denotes the set of Borel sets and $\mathcal{M}(W)$ denotes the set of Borel, finite, signed measures on $W$.
\section{Preliminaries}
We start by formally defining countably piecewise linear.
\begin{definition}
    A \textbf{convex polyhedron} $C$ is a subset of $\R^{n}$ such that $C=\bigcap_{H\in\mathcal{H}} H$ where $\mathcal{H}$ is a finite set of closed half-spaces.
    A \textbf{supporting hyperplane} $\mathfrak{h}$ of $C$ is the boundary of a half-space in $\mathcal{H}$ such that $C\cap \mathfrak{h}\not=\emptyset$.
    A \textbf{face} of $C$ is a set of the form $C\cap \mathfrak{h}$ where $\mathfrak{h}$ is a supporting hyperplane.
\end{definition}
\begin{remark}
    The union of the faces of a polyhedron $C$ form the boundary of $C$.
\end{remark}
\begin{definition}
    A \textbf{continuous countably (finitely) piecewise linear function} is a continuous function such that there is a countable (finite) collection of convex polyhedra that cover the domain where the function is affine when restricted to each polyhedron.
\end{definition}
\begin{remark}
    The requirement of continuity in the definition does not impose any limitations on the results.
    Every function with a representation of the form $f(\vecx{x})=\int_{\mathbb{S}^n\times\R} \sigma(\veca{a}\cdot\vecx{x}-b)-\sigma(-b)\dif\mu(\veca{a},b)+c_0$ is continuous.
\end{remark}
The basis of the argument is that while hyperplanes can be induced by many different measures, the creases associated with boundaries in a piecewise linear function can only be created by point masses.
To isolate the parts of the measure that induce non-affineness, the measure is associated with a measure on the projective $n$-sphere cross $\R$.
Similar to this procedure, Ongie et al.
in \cite{ongieFunctionSpaceView2019} decomposed measures in $\mathcal{M}(\mathbb{S}^n\times\R)$ into even and odd components, where the odd component induced an affine function and the even component was unique.
We use $\mathcal{S}^n$ as a representation of the projective $n$-sphere.

The following lemmas reduce the problem to measures in $\mathcal{M}(\mathcal{S}^n\times\R)$.
This reduction is possible for any compactly supported measure in $\mathcal{M}(\R^{n+1}\times\R)$ or any measure in $\mathcal{M}(\mathbb{S}^n\times\R)$.
\begin{lemma}
    \label{appliesToCompactThingsToo}
    Suppose $\tau$ is a compactly supported measure in $\mathcal{M}(\R^{n+1}\times\R)$.
    Then, there exists $\mu\in\mathcal{M}(\mathbb{S}^n\times\R)$ such that for all $\vecx{x}\in\R^{n+1}$,
    \[
        \int_{\R^{n+1}\times\R} \sigma\left(\veca{a}\cdot\vecx{x} -b \right)-\sigma(-b)\dif\tau(\veca{a},b) = \int_{\mathbb{S}^n\times\R} \sigma\left(\veca{a}\cdot\vecx{x} -b \right)-\sigma(-b)\dif\mu(\veca{a},b).
    \]
\end{lemma}
\begin{proof}
    Let $g:\left(\R^{n+1}\setminus\{\veca{0}\}\right)\times\R\rightarrow \mathbb{S}^n\times\R$ be defined as $g\left(\veca{a},b\right)=\left( \frac{\veca{a}}{\|\veca{a}\|},\frac{b}{\|\veca{a}\|} \right)$.
    Let $\mu_1$ be the Borel measure defined as $\mu_1(E)=\int_E \|\veca{a}\| \dif\tau\left(\veca{a},b\right)$.
    Since $\tau$ has compact support and is finite, $|\mu_1||(\R^{n+1}\times\R)<\infty$ and $\mu_1\in \mathcal{M}(\R^{n+1}\times\R)$.
    Let $\mu\coloneqq \mu_1\circ g^{-1}$.
    Then,
    \begin{align*}
         & \int_{\R^{n+1}\times\R} \sigma\left(\veca{a}\cdot\vecx{x} -b \right)-\sigma(-b)\dif\tau\left(\veca{a},b\right)                                                                                                                                      \\
         & = \int_{\left(\R^{n+1}\setminus\{\veca{0}\}\right)\times\R} \|\veca{a}\|\left(\sigma\left( \frac{\veca{a}}{\|\veca{a}\|}\cdot x - \frac{b}{\|\veca{a}\|} \right)-\sigma\left( -\frac{b}{\|\veca{a}\|} \right)\right)\dif\tau\left(\veca{a},b\right) \\
         & \quad +\int_{\{\veca{0}\}\times\R}\sigma\left(-b \right)-\sigma(-b)\dif\tau\left(\veca{a},b\right)                                                                                                                                                  \\
         & = \int_{\left(\R^{n+1}\setminus\{\veca{0}\}\right)\times\R} \sigma\left( \frac{\veca{a}}{\|\veca{a}\|}\cdot \vecx{x} - \frac{b}{\|\veca{a}\|} \right)-\sigma\left( -\frac{b}{\|\veca{a}\|} \right)\dif\mu_1\left(\veca{a},b\right)+0                \\
         & = \int_{\mathbb{S}^n\times\R} \sigma\left( \veca{a}\cdot\vecx{x} - b \right)-\sigma(-b)\dif \mu\left(\veca{a},b\right).
    \end{align*}
\end{proof}
\begin{lemma}
    \label{packingDownIntoHalfsphere}
    Suppose $\tau\in\mathcal{M}\left(\mathbb{S}^n\times \R\right)$.
    Then, there exists $\mu\in\mathcal{M}\left(\mathcal{S}^n\times\R\right)$ and $\veca{{a}}_0\in\R^{n+1}$ such that for all $\vecx{x}\in\R^{n+1}$,
    \[
        \int_{\mathbb{S}^n\times \R} \sigma\left(\veca{a}\cdot \vecx{x} -b \right) -\sigma(-b)\dif\tau\left(\veca{a},b\right)= \int_{\mathcal{S}^n\times\R} \sigma\left(\veca{a}\cdot \vecx{x} -b \right)-\sigma(-b)\dif \mu\left(\veca{a},b\right) + \veca{{a}}_0\cdot \vecx{x}.
    \]
\end{lemma}
\begin{proof}
    Let $g\left(\veca{a},b\right)=\left(-\veca{a},-b \right)$ on $\mathcal{S}^n\times \R$.
    Let $\mu\coloneqq \tau+\tau\circ g^{-1}$ and $\veca{a}_0\coloneqq -\int_{\mathcal{S}^n\times \R}\veca{a}\dif\tau\circ g^{-1}\left(\veca{a},b\right) $.
    Note, $\sigma\left(-x\right)=\sigma\left(x\right)-x$.
    It follows
    \begin{align*}
         & \int_{\mathbb{S}^n\times\R} \sigma\left(\veca{a}\cdot \vecx{x} -b \right)-\sigma(-b)\dif\tau\left(\veca{a},b\right)                                                                                                                              \\
         & = \int_{\mathcal{S}^n\times \R} \sigma\left(\veca{a}\cdot \vecx{x} -b \right)-\sigma(-b)\dif\tau\left(\veca{a},b\right) + \int_{-\mathcal{S}^n\times \R} \sigma\left(\veca{a}\cdot \vecx{x} -b \right)-\sigma(-b)\dif\tau\left(\veca{a},b\right) \\
         & = \int_{\mathcal{S}^n\times \R} \sigma\left(\veca{a}\cdot \vecx{x} -b \right)-\sigma(-b)\dif\tau\left(\veca{a},b\right)                                                                                                                          \\
         & \quad + \int_{\mathcal{S}^n\times \R} \sigma\left(-\veca{a}\cdot \vecx{x} + b\right)-\sigma(b)\dif\tau\circ g^{-1}\left(\veca{a},b\right)                                                                                                        \\
         & = \int_{\mathcal{S}^n\times \R} \sigma\left(\veca{a}\cdot \vecx{x} -b \right)-\sigma(-b)\dif\tau\left(\veca{a},b\right)                                                                                                                          \\
         & \quad + \int_{\mathcal{S}^n\times \R} \sigma\left(\veca{a}\cdot \vecx{x} - b\right)-\left(\veca{a}\cdot \vecx{x} -b \right)-\sigma(-b)-b\dif\tau\circ g^{-1}\left(\veca{a},b\right)                                                              \\
         & = \int_{\mathcal{S}^n\times \R} \sigma\left(\veca{a}\cdot \vecx{x} -b \right)-\sigma(-b)\dif\tau\left(\veca{a},b\right)                                                                                                                          \\
         & \quad + \int_{\mathcal{S}^n\times \R} \sigma\left(\veca{a}\cdot \vecx{x} - b\right)-\sigma(-b)\dif\tau\circ g^{-1}\left(\veca{a},b\right)-\int_{\mathcal{S}^n\times\R} \veca{a}\cdot \vecx{x}\dif\tau\circ g^{-1}\left(\veca{a},b\right)         \\
         & = \int_{\mathcal{S}^n\times \R} \sigma\left(\veca{a}\cdot \vecx{x} -b \right)-\sigma(-b)\dif(\tau+\tau\circ g^{-1})\left(\veca{a},b\right) - \int_{\mathcal{S}^n\times \R}\veca{a}\dif\tau\circ g^{-1}\left(\veca{a},b\right) \cdot \vecx{x}     \\
         & =\int_{\mathcal{S}^n\times\R} \sigma\left(\veca{a}\cdot \vecx{x} -b \right)-\sigma(-b)\dif \mu\left(\veca{a},b\right) + \veca{{a}}_0\cdot \vecx{x}.
    \end{align*}

\end{proof}
The main result will be a consequence of the following theorem, which will be proved by induction.
\restate{mainTheorem}

\noindent We will utilize that first order directional derivatives are constant on affine polyhedra.
To simplify calculations, we will be particularly interested in directional derivatives in the direction $\eLast$.
\begin{lemma}
    \label{directionalDerivativeCalc}
    Suppose $f(\vecx{x})=\int_{\mathcal{S}^n\times\R} \sigma\left(\veca{a}\cdot \vecx{x} -b \right)-\sigma(-b)\dif\mu\left(\veca{a},b\right)$ where $\mu\in\mathcal{M}\left(\mathcal{S}^n\times\R\right)$.
    Then,
    \[
        D_{\eLast^+} f(\vecx{x})=\int_{\left\{\left(\veca{a},b\right)\in\mathcal{S}^n\times\R \; | \; \veca{a}\cdot \vecx{x} \geq b\right\} } \veca{a}\cdot \eLast \dif\mu\left(\veca{a},b\right).
    \]
\end{lemma}
\begin{proof}
    First,
    \begin{align*}
         & \lim_{h\rightarrow 0^+}\frac{f(\vecx{x}+h\eLast)-f(\vecx{x})}{h}                                                                                                                                              \\
         & =\lim_{h\rightarrow 0^+}\int_{\mathcal{S}^n\times\R}\frac{\sigma\left(\veca{a}\cdot \vecx{x} -b+\veca{a}\cdot(h\eLast)\right)-\sigma\left(\veca{a}\cdot \vecx{x}-b \right)}{h}\dif\mu\left(\veca{a},b\right).
    \end{align*}
    Whenever $\veca{a}\cdot \vecx{x} < b$, for sufficiently small $h$,
    \[
        \sigma\left(\veca{a}\cdot \vecx{x} -b+\veca{a}\cdot (h\eLast)\right)=\sigma\left(\veca{a}\cdot \vecx{x}-b \right)=0.
    \]
    Thus, when $\veca{a}\cdot \vecx{x} < b$,
    \[
        \lim_{h\rightarrow 0^+}\frac{\sigma\left(\veca{a}\cdot \vecx{x} -b+\veca{a}\cdot (h\eLast)\right)-\sigma\left(\veca{a}\cdot \vecx{x}-b \right)}{h}=0.
    \]
    By definition of $\mathcal{S}^n$, $\veca{a}\cdot(h\eLast)\geq 0$ when $h\geq 0$ for all $\veca{a}\in\mathcal{S}^n$.
    Hence, if $\veca{a}\cdot \vecx{x} -b\geq 0$, then $\veca{a}\cdot \vecx{x} -b+\veca{a}\cdot(h\eLast)\geq 0$ for all $\veca{a}\in\mathcal{S}^n$ and $h\geq 0$.
    It follows whenever $\veca{a}\cdot \vecx{x} \geq b$,
    \begin{align*}
         & \lim_{h\rightarrow 0^+}\frac{\sigma\left(\veca{a}\cdot\vecx{x} -b+\veca{a}\cdot (h\eLast)\right)-\sigma\left(\veca{a}\cdot\vecx{x}-b \right)}{h}                                                        \\
         & =\lim_{h\rightarrow 0^+}\frac{\veca{a}\cdot\vecx{x} -b+\veca{a}\cdot (h\eLast)-\left(\veca{a}\cdot\vecx{x}-b \right)}{h}=\lim_{h\rightarrow 0^+}\frac{\veca{a}\cdot (h\eLast)}{h}=\veca{a}\cdot \eLast.
    \end{align*}
    Further, as $\| \veca{a}\|=\|\eLast\|=1$, for all $\veca{a},\vecx{x},b$ and all $h\geq 0$,
    \begin{equation}
        \left\lvert \frac{\sigma\left(\veca{a}\cdot \vecx{x} -b+\veca{a}\cdot(h\eLast)\right)-\sigma\left(\veca{a}\cdot \vecx{x}-b \right)}{h}\right\rvert \leq \frac{|\veca{a}\cdot(h\eLast)|}{h}\leq 1.
        \label{dominatedEquation}
    \end{equation}
    Since $|\mu|(\mathcal{S}^n\times\R)<\infty$, by \autoref{dominatedEquation}, Dominated Convergence Theorem applies.
    Therefore,
    \begin{align*}
         & \lim_{h\rightarrow 0^+}\frac{f(\vecx{x}+(h\eLast))-f(\vecx{x})}{h}                                                                                                                                          \\
         & =\int_{\mathcal{S}^n\times\R}\lim_{h\rightarrow 0^+}\frac{\sigma\left(\veca{a}\cdot\vecx{x} -b+\veca{a}\cdot (h\eLast)\right)-\sigma\left(\veca{a}\cdot\vecx{x}-b \right)}{h}\dif\mu\left(\veca{a},b\right) \\
         & = \int_{\left\{\left(\veca{a},b\right)\in\mathcal{S}^n\times\R\; | \; \veca{a}\cdot\vecx{x} \geq b \right\}} \lim_{h\rightarrow 0^+}\frac{\veca{a}\cdot (h\eLast)}{h}\dif \mu\left(\veca{a},b\right)        \\
         & = \int_{\left\{\left(\veca{a},b\right)\in\mathcal{S}^n\times\R\; | \; \veca{a}\cdot\vecx{x} \geq b \right\}} \veca{a}\cdot \eLast \dif \mu\left(\veca{a},b\right).
        \qedhere
    \end{align*}
\end{proof}

\noindent We prove the one-dimensional case of \autoref{mainTheorem} to serve as a base case for induction.
\begin{lemma}
    \label{onePointDoesntDoIt}
    Suppose $\mu\in\mathcal{M}(\mathcal{S}^0\times\R)$ is such that
    \begin{enumerate}
        \item $\mu$ is atomless
        \item $f(x)=\int_{\mathcal{S}^0\times\R} \sigma\left(ax -b \right)-\sigma(-b)\dif\mu\left(a,b\right)$ is a continuous countably piecewise linear function.
    \end{enumerate}
    Then, $\mu$ is the zero measure.
\end{lemma}
\begin{proof}
    Since $|\mathcal{S}^0|=|\{1\}|=1$, $\mu$ is uniquely determined by its marginal measure on $\R$, $\mu_\R$.
    As $f$ is countably piecewise linear, there are countably many intervals $\{[q_i,r_i]\}_{i\in\N}$ such that $f$ is affine when restricted to each interval and $\R\setminus\left(\bigcup_{i\in\N} (q_i,r_i)\right)$ is countable.
    Further, $\mu$ is atomless, so $\mu_{\R}$ is determined by its values on closed intervals that are subsets of intervals in $\{(q_i,r_i)\}_{i\in\N}$ \cite{jacksonSclassGeneratedOpen1999}.
    Thus, it suffices to show $\mu_{\R}([x_1,x_2])=0$ whenever $f$ is affine on $(q,r)$ and $[x_1,x_2]\subseteq (q,r)$.
    \\
    Suppose $f$ is affine on $(q,r)$ and $[x_1,x_2]\subseteq (q,r)$.
    Note,
    \begin{align*}
        f'(x_2) & =\int_{\left\{\left(a,b\right)\in\mathcal{S}^0\times\R\; | \; x_2 \geq b\right\}} a\cdot 1\dif\mu\left(a,b\right) = \int_{\left\{\left(a,b\right)\in\mathcal{S}^0\times\R\; |\; x_2 \geq b\right\}} 1\cdot 1\dif \mu\left(a,b\right) \\
                & = \int_{\left\{\left(a,b\right)\in\mathcal{S}^0\times\R\; |\; x_2 \geq b\right\}} 1 \dif \mu\left(a,b\right) = \mu\left( {\left\{\left(a,b\right)\in\mathcal{S}^0\times\R\; | \;x_2 \geq b\right\}} \right)
    \end{align*}
    Similarly,
    \[
        f'(x_1) = \mu\left( \left\{\left(a,b\right)\in\mathcal{S}^0\times\R\; |\; x_1 \geq b\right\}\right).
    \]
    Therefore, as $f$ is affine in between $x_1$ and $x_2$,
    \[
        0=f'(x_2)-f'(x_1)=\mu\left(\{1\}\times \big(x_1,x_2\big]\right)=\mu\left(\left\{1\right\}\times \big[x_1,x_2\big]\right)=\mu_\R([x_1,x_2]).
    \]

\end{proof}
\section{Constructing a Dense Set of Directions}
Any set of countably many points has zero weight with respect to an atomless measure.
In the one-dimensional case, this allows us to disregard points in $\mathcal{S}^0\times\R$ associated with boundaries when determining $\mu$ is the zero measure.
However, in higher dimensions, there are more than countably many points associated with boundaries.
Nonetheless, a carefully picked subset of points associated with boundaries will have zero weight with respect to $\mu$ and will be large enough to ultimately conclude $\mu$ is the zero measure.
The first step to constructing this set is finding a large set of non-co-hyperplanar points.
\begin{proposition}
    \label{superUsefulSet}
    For every $n\in\N$, there exists a set $S\subseteq\R^n$ such that
    \begin{enumerate}
        \item For every open ball $B\subseteq \R^n$, $S\cap B$ is uncountable
        \item For every hyperplane $P\subseteq \R^n$, $|S\cap P|\leq n$.
    \end{enumerate}
\end{proposition}
\begin{proof}
    By \cite{neumannSystemAlgebraischUnabhaengiger1928}, there is a set $I\subseteq\R$ algebraically independent over $\Q$ such that $|I|=|\R|$.
    \\
    There is a bijective function $\phi: [n]\times\N\times \R \rightarrow I$.\\
    Let $\{B_m\}_{m\in\N}$ be an enumeration of open balls in $\R^n$ centered at rational coordinates with rational radius.\\
    Note, $0\not\in I$.
    For every $m\in\N,r\in\R$, there exist $q_{1,m,r},\ldots,q_{n,m,r}\in\Q\setminus\{0\}$ such that
    \[
        \left(q_{1,m,r}\phi(1,m,r),\ldots,q_{n,m,r}\phi(n,m,r)\right)\in B_m.
    \]
    The set $\{q_{\ell,m,r}\phi(\ell,m,r) \; | \;\ell\in[n], m\in\N,r\in\R\}$ is also algebraically independent and each element has a unique representation of the form $q_{\ell,m,r}\phi(\ell,m,r)$.
    Define
    \[
        S\coloneqq\{(q_{1,m,r}\phi(1,m,r),\ldots,q_{n,m,r}\phi(n,m,r))\; | \; m\in\N,r\in\R\}.
    \]
    Consider an open ball $B\subseteq \R^n$.
    Since $\Q$ is dense, there is $m_0$ such that $B_{m_0}\subseteq B$.
    Further, $\{(q_{1,m_0,r}\phi(1,m_0,r),\ldots,q_{n,m_0,r}\phi(n,m_0,r))\; | \; r\in\R\}\subseteq B_{m_0}\subseteq B$.
    It follows $S\cap B$ is uncountable.
    \\
    By way of contradiction, suppose there exist distinct $\left(z_0^1,\ldots,z_0^n\right),\ldots,\left(z_{n}^1,\ldots,z_{n}^n\right)\in S\cap P$ for some hyperplane $P$.
    It follows any $n$ vectors between these points are linearly dependent, so
    \begin{equation}
        \det
        \begin{bmatrix}
            z_0^1-z_1^1 & \ldots & z_0^n-z_1^n \\
            \vdots      & \ddots & \vdots      \\
            z_0^1-z_n^1 & \ldots & z_0^n-z_n^n
        \end{bmatrix}
        =0.
        \label{keyDeterminant}
    \end{equation}
    The determinant is a polynomial over $\Q$ in terms of $z_i^j$.
    Since a unique $z_i^j$, $i\geq 1$, is an addend in each entry, the determinant cannot be the trivial polynomial.
    This contradicts the $z_i^j$ being algebraically independent.
    \\
    It follows for all hyperplanes $P$, $|S\cap P|\leq n$.
\end{proof}
\begin{corollary}
    Suppose $S\subseteq \R^{n}$ is as in \autoref{superUsefulSet}.
    Let $\phi:S\rightarrow \R$ and $S'\coloneqq \{(\vecz{z},\phi(\vecz{z}))\; | \; \vecz{z}\in S\}\subseteq\R^{n+1}$.
    For every $n-1$ dimensional affine subspace $P\subseteq \R^{n+1}$, $|S'\cap P|\leq n$.
    \label{projectingUsefulSets}
\end{corollary}
\begin{proof}
    Let $\phi:S\rightarrow \R$.
    By way of contradiction, let $P$ be a $n-1$ dimensional affine subspace and suppose distinct points $(z_0^1,\ldots,z_{0}^{n+1}),\ldots,(z_{n}^1,\ldots,z_{n}^{n+1})\in S'\cap P.
    $
    Then,
    \[
        \rank
        \begin{bmatrix}
            z_0^1-z_1^1 & \ldots & z_0^{n}-z_1^{n} \\
            \vdots      & \ddots & \vdots          \\
            z_0^1-z_n^1 & \ldots & z_0^{n}-z_n^{n}
        \end{bmatrix}
        \leq \rank
        \begin{bmatrix}
            z_0^1-z_1^1 & \ldots & z_0^{n+1}-z_1^{n+1} \\
            \vdots      & \ddots & \vdots              \\
            z_0^1-z_n^1 & \ldots & z_0^{n+1}-z_n^{n+1}
        \end{bmatrix}
        \leq n-1.
    \]
    Therefore, as in \autoref{keyDeterminant} of \autoref{superUsefulSet},
    \[
        \det
        \begin{bmatrix}
            z_0^1-z_1^1 & \ldots & z_0^{n}-z_1^{n} \\
            \vdots      & \ddots & \vdots          \\
            z_0^1-z_n^1 & \ldots & z_0^{n}-z_n^{n}
        \end{bmatrix}
        =0,
    \]
    a contradiction.
\end{proof}
\section{Proof of Main Theorem}
\begin{lemma}
    Let $W$ be a metric space and $\mu\in\mathcal{M}(W)$.
    Consider a collection of sets $\mathcal{P}\subseteq\mathcal{B}(W)$ such that there exists a $c\in \N$ where $|\mu|\left( \bigcap_{i\in [c]} P_i \right)=0$ for all distinct $P_1,\ldots,P_c\in \mathcal{P}$.
    Then, there are only countably many $P\in \mathcal{P}$ such that $|\mu|\left(P\right)>0$.
    \label{LimitTheIntersectionLimitTheWeighty}
\end{lemma}
\begin{proof}
    Every uncountable family of sets of positive measure has an infinite subfamily with positive intersection \cite{halmosLargeIntersectionsLarge1992}.
    The lemma follows from the contrapositive.
\end{proof}
\begin{definition}
    Suppose $\vecz{z}_0\in\R^n,y_1,y_2\in\R\cup\{\pm\infty\}$ with $y_1\leq y_2$.
    Define
    \[
        \overbar{L}_{\vecz{z}_0}(y_1)\coloneqq\left\{ \left(\vecu{u},v\right)\in\R^n\times\R \; | \; v-\vecu{u}\cdot \vecz{z}_0 = y_1\right\}
    \]
    and
    \[
        L_{\vecz{z}_0}(y_1,y_2)\coloneqq\left\{ \left(\vecu{u},v\right)\in\R^{n}\times\R \; | \; v-\vecu{u}\cdot \vecz{z}_0 \in (y_1,y_2]\right\}.
    \]
\end{definition}
\begin{restatethis}{theorem}{mainTheorem}
    Suppose $\mu\in\mathcal{M}\left(\mathcal{S}^n\times\R\right)$ is such that
    \begin{enumerate}
        \item $\mu$ is atomless
        \item $f(\vecx{x})=\int_{\mathcal{S}^n \times \R} \sigma\left(\veca{a}\cdot\vecx{x} -b \right)-\sigma(-b)\dif\mu\left(\veca{a},b\right)$ is a continuous countably piecewise linear function from $\R^{n+1}$ to $\R$.
    \end{enumerate}
    Then, $\mu$ is the zero measure.
\end{restatethis}
\begin{proof}
    The proof is separated into the following parts:
    \begin{enumerate}
        \item \textbf{Definitions and Maps Between Measure Spaces}
        \item \textbf{Refinement of $S$}
        \item \textbf{Vanishing Integrals over Line Segments }
        \item \textbf{Vanishing Integrals over Half-Spaces}
        \item \textbf{Conclusion with Radon-Nikodym}.
    \end{enumerate}
    \textbf{Definitions and Maps Between Measure Spaces}\\
    First, \autoref{onePointDoesntDoIt} proves the theorem for the case $n=0$.
    For induction, assume the theorem holds for $n-1$.
    \\
    Suppose $\mu\in\mathcal{M}\left(\mathcal{S}^n\times\R\right)$ satisfies all the hypotheses.\\
    Let $\mathcal{C}$ be a countable collection of convex polyhedra that cover the domain of $f$ such that $f$ is affine on each.
    Let $\mathcal{F}_1$ be the set of faces of polyhedra in $\mathcal{C}$ with normal vector not orthogonal to $\eLast$.
    \\
    Define the following sets of hyperplanes in $\R^{n+1}$
    \[
        \mathcal{H}_0\coloneqq \left\{ \hyp{a}{b} \; | \; \veca{a}\in\mathcal{S}^n, b\in\R\right\}\quad\text{ and } \quad \mathcal{H}_1\coloneqq \left\{ \hyp{a}{b} \; | \; \veca{a}\in\mathcal{S}^n, \veca{a}\cdot \eLast\not=0, b\in\R\right\}.
    \]
    Define the map $\gamma:\mathcal{S}^n\times\R\rightarrow \mathcal{H}_0$ as $\gamma\left(\veca{a},b\right)=\hyp{a}{b}$.
    By construction of $\mathcal{S}^n$, this is bijective.
    \\
    Define the map $\psi:\mathcal{H}_1\rightarrow\R^{n}\times\R$ such that for $\veca{a}=(a_1,\ldots,a_{n+1})$,
    \begin{equation}
        \psi\left(\hyp{a}{b}\right) =\left( \frac{a_1}{a_{n+1}},\ldots,\frac{a_n}{a_{n+1}}, \frac{b}{a_{n+1}} \right)
        \label{DomainToUV}.
    \end{equation}
    By definition of $\mathcal{H}_1$, it is routine to verify $\psi$ is well-defined and bijective.
    \\
    The image of $\psi$ is $\R^{n}\times\R$, however, elements in the image of $\psi$ should \textit{not} be thought of as being in the domain of $f$.
    Therefore, identify generic elements in the image of $\psi$ with $\left(\vecu{u},v\right)\in\R^{n}\times\R$ and call the space $\Xi\times V$ where $\Xi=\R^n,$ $V=\R$.
    \\
    Define $\varphi:\gamma^{-1}[\mathcal{H}_1]\rightarrow \Xi\times V$ as $\varphi\coloneqq \psi\circ \gamma$.
    Then, $\mu\circ\varphi^{-1}$ is a measure on $\Xi\times V$.
    Since $\varphi$ is bijective, $\mu\circ\varphi^{-1}$ is atomless.
    \\
    For fixed $\vecz{z}_0\in\R^{n}$, $y_0\in\R$,
    \begin{equation}
        \label{PlanesThroughPointsArePlanes}
        \psi\left[ \left\{ \hyp{a}{b} \in \mathcal{H}_1\; | \; \veca{a}\cdot (\vecz{z}_0,y_0) = b\right\} \right]=\left\{\left(\vecu{u},v\right)\in \Xi\times V \; | \; v = y_0 + \vecu{u}\cdot \vecz{z}_0\right\}.
    \end{equation}
    That is the image under $\psi$ of hyperplanes in $\mathcal{H}_1$ which intersect $(\vecz{z}_0,y_0)\in\R^n\times\R$ is a hyperplane in $\Xi\times V$.
    \\
    \textbf{Refinement of $S$}\\
    Let $S\subseteq\R^n$ be the set in \autoref{superUsefulSet}.\\
    Suppose $\mathfrak{h}\in\mathcal{H}_1$.
    Let $\phi_{\mathfrak{h}}:\R^n\rightarrow \R$ be the unique function such that $(\vecz{z},\phi_\mathfrak{h}(\vecz{z}))\in \mathfrak{h}$ for all $\vecz{z}\in \R^n$.
    \\
    Let $P_{\vecz{z},\mathfrak{h}} = \left\{\left(\vecu{u},v\right)\in \Xi\times V\; | \; v = \phi_{\mathfrak{h}}(\vecz{z}) + \vecu{u}\cdot \vecz{z} \right\}$.
    By \autoref{PlanesThroughPointsArePlanes} and because $\psi$ is injective, all hyperplanes in $\psi^{-1}[P_{\vecz{z},\mathfrak{h}}]$ intersect the point $(\vecz{z},\phi_{\mathfrak{h}}(\vecz{z}))$ in the domain.
    \\
    Let $\mathcal{P}_\mathfrak{h} = \{ P_{\vecz{z},\mathfrak{h}}\; | \;  \vecz{z}\in S\}$.\\
    For unique $\vecz{z}_1,\ldots, \vecz{z}_{n+1}\in S$, consider $\bigcap_{i\in [n+1]}
        P_{\vecz{z}_i,\mathfrak{h}}$.
    It follows any hyperplane in $ \psi^{-1}\left[ \bigcap_{i\in [n+1]}P_{\vecz{z}_i,\mathfrak{h}}\right]$ intersects the points $\{(\vecz{z}_1,\phi_\mathfrak{h}(\vecz{z}_1)),\ldots,(\vecz{z}_{n+1},\phi_\mathfrak{h}(\vecz{z}_{n+1}))\}$ where each $\vecz{z}_i\in S$.
    By \autoref{projectingUsefulSets}, these points do not lie on a common $n-1$ dimensional affine subspace, so $\mathfrak{h}$ is the only hyperplane in the domain of $f$ intersecting $\{(\vecz{z}_1,\phi_\mathfrak{h}(\vecz{z}_1)),\ldots,(\vecz{z}_{n+1},\phi_\mathfrak{h}(\vecz{z}_{n+1}))\}$.
    It follows $\bigcap_{i\in [n+1]}P_{\vecz{z}_i,\mathfrak{h}} =\{\psi(\mathfrak{h})\}$.
    Since $\mu\circ\varphi^{-1}$ is atomless, $|\mu\circ\varphi^{-1}|\left(\bigcap_{i\in [n+1]}P_{\vecz{z}_i,\mathfrak{h}} \right)=0$.
    \\
    By \autoref{LimitTheIntersectionLimitTheWeighty},
    there are only countably many $P_{\vecz{z},\mathfrak{h}}\in \mathcal{P}_\mathfrak{h}$ such that $|\mu\circ\varphi^{-1}|\left(P_{\vecz{z},\mathfrak{h}}\right)>0$.\\ Define $S_{\mathfrak{h}}\coloneqq \{\vecz{z}\in S \; | \; |\mu\circ\varphi^{-1}|\left(P_{\vecz{z},\mathfrak{h}}\right) = 0 \}, $ so $S\setminus S_{\mathfrak{h}}$ is countable.\\
    Let $\mathcal{H}_{supp}$ be the set of supporting hyperplanes of polyhedra in $\mathcal{C}$.
    Consider
    \[
        S'\coloneqq \bigcap_{\mathfrak{h}\in \mathcal{H}_1\cap \mathcal{H}_{supp}} S_{\mathfrak{h}}.
    \]
    Since $\mathcal{H}_{supp}$ is countable, $S\setminus S'$ is countable.
    Since $B\cap S$ is uncountable for all open balls $B\subseteq \R^n$, $S'$ is dense in $\R^n$.
    Notice, whenever $\vecz{z}\in S'$ and $(\vecz{z},y)$ is on a hyperplane in $\mathcal{H}_1\cap \mathcal{H}_{supp}$,
    \begin{equation}
        |\mu\circ\varphi^{-1}|\left(\left\{\left(\vecu{u},v\right)\in \Xi\times V \; | \; v = y + \vecu{u}\cdot \vecz{z}\right\} \right)=0.
        \label{zeroAboveSPrimeOnBoundaries}
    \end{equation}
    \textbf{Vanishing Integrals over Line Segments}\\
    Suppose $\vecz{z}_0\in\R^n$, $y_1,y_2\in\R$.
    Suppose $y_1\leq y_2$.
    \\
    By \autoref{PlanesThroughPointsArePlanes}, $L_{\vecz{z}_0}(y_1,y_2)$ is the image under $\psi$ of hyperplanes in $\mathcal{H}_1$ which intersect the line segment between $(\vecz{z}_0,y_1)$ (exclusive) and $(\vecz{z}_0,y_2)$ (inclusive).
    Therefore,
    \[
        \varphi^{-1}\left[L_{\vecz{z}_0}(y_1,y_2)\right]=\left\{\left(\veca{a},b\right)\in\mathcal{S}^n\times\R \; |\; \exists y'\in (y_1,y_2]\;\; \veca{a}\cdot(\vecz{z}_0,y')=b, \;\veca{a}\cdot\eLast\not=0\right\}.
    \]
    If $\varphi\left(\veca{a},b\right)=\left(\vecu{u},v\right)$, then
    \[
        \frac{1}{\sqrt{1+\sum_{i\in [n]}\xi_i^2}}=\frac{1}{\sqrt{1+\sum_{i\in [n]}\frac{a_i^2}{a_{n+1}^2}}}=\frac{a_{n+1}}{\sqrt{\sum_{i\in [n+1]}a_i^2}}=a_{n+1}=\veca{a}\cdot \eLast.
    \]
    Therefore, as $\left(\veca{a},b\right)$ such that $\veca{a}\cdot\eLast=0$ do not contribute to the integral,
    \begin{align*}
         & \int_{L_{\vecz{z}_0}(y_1,y_2)}\frac{1}{\sqrt{1+\| \vecu{u}\|^2}} \dif\mu\circ\varphi^{-1}\left(\vecu{u},v\right) = \int_{\varphi^{-1}\left[L_{\vecz{z}_0}(y_1,y_2)\right]} \veca{a}\cdot \eLast\dif\mu\left(\veca{a},b\right) \\
         & =\int_{\left\{\left(\veca{a},b\right)\in\mathcal{S}^n\times\R \; |\;\exists y' \in (y_1,y_2] \;\; \veca{a}\cdot(\vecz{z}_0,y')= b\right\}} \veca{a}\cdot\eLast\dif\mu\left(\veca{a},b\right).
    \end{align*}
    Recall, for $y\in\R$,
    \[
        D_{\eLast^+} f(\vecz{z}_0,y)=\int_{\left\{\left(\veca{a},b\right)\in\mathcal{S}^n\times\R \; | \; \veca{a}\cdot( \vecz{z}_0,y) \geq b\right\} } \veca{a}\cdot \eLast \dif\mu\left(\veca{a},b\right).
    \]
    By definition of $\mathcal{S}^n$, $y\mapsto \veca{a}\cdot (\vecz{z}_0,y)$ is a non-decreasing, continuous function on $\R$ for any fixed $\veca{a}\in\mathcal{S}^n$.
    Thus, $\veca{a}\cdot (\vecz{z}_0,y_2)\geq \veca{a}\cdot (\vecz{z}_0,y_1)$ for all $\veca{a}\in\mathcal{S}^n$.
    Further, $\veca{a}\cdot (\vecz{z}_0,y_2)\geq b$ and $\veca{a}\cdot (\vecz{z}_0,y_1)<b$ if and only if $\veca{a}\cdot (\vecz{z}_0,y')=b$ for some $y'\in (y_1,y_2]$.
    Therefore,
    \[
        D_{\eLast^+} f(\vecz{z}_0,y_2)-D_{\eLast^+} f(\vecz{z}_0,y_1) = \int_{L_{\vecz{z}_0}(y_1,y_2)} \frac{1}{\sqrt{1+\|\vecu{u}\|^2}} \dif \mu\circ\varphi^{-1}\left(\vecu{u},v\right).
    \]
    It follows whenever $D_{\eLast^+} f(\vecz{z}_0,y_1)=D_{\eLast^+} f(\vecz{z}_0,y_2)$,
    \begin{equation}
        \int_{L_{\vecz{z}_0}(y_1,y_2)} \frac{1}{\sqrt{1+\|\vecu{u}\|^2}} \dif \mu\circ\varphi^{-1}\left(\vecu{u},v\right)=0.
        \label{equalityOnLineImpliesZero}
    \end{equation}
    \textbf{Vanishing Integrals over Half-Spaces}\\
    Consider $\vecz{z}_0\in S'$.
    Consider an interval $(y_0,\infty)\subseteq \R$.
    \\
    For sets $E\subseteq \R^{n+1}$, let $\text{ri}_{\vecz{z}_0}(E)$ denote the relative interior of $E\cap \left( \{\vecz{z}_0\}\times (-\infty,\infty) \right)$ with respect to $\{\vecz{z}_0\}\times (-\infty,\infty)$.
    Then, define
    \[
        J\coloneqq \left\{y\in (y_0,\infty) \; | \; (\vecz{z}_0,y)\in \bigcup_{C\in\mathcal{C}} \text{ri}_{\vecz{z}_0}\left( C\right)\right\}.
    \]
    It follows $J$ is open.
    Then, there are countably many $q_i,r_i\in \R\cup\{\pm\infty\}$ such that $J=\bigcup_{i\in \N} (q_i,r_i)$, the intervals pairwise disjoint.
    \\
    Moreover, $D_{\eLast^+} f(\vecz{z},y)$ is constant on $\text{ri}_{\vecz{z}_0}\left( C   \right)$ for every $C\in\mathcal{C}$.
    As locally constant functions are constant on connected components, $D_{\eLast^+} f(\vecz{z},y)$ is constant on $\{\vecz{z}_0\}\times(q_i,r_i)$ for all $i\in \N$.
    By \autoref{equalityOnLineImpliesZero}, for all $m\in\N$,
    \begin{equation}
        \int_{ L_{\vecz{z}_0}\left(q_{i}+\frac{1}{m},r_{i}-\frac{1}{m}\right)} \frac{1}{\sqrt{1+\left\|\vecu{u}\right\|^2}}\dif\mu\circ\varphi^{-1}\left(\vecu{u},v\right) = 0.
        \label{linesSegmentsZero}
    \end{equation}
    Thus, define $E_m \coloneqq\bigcup_{i\in \N} L_{\vecz{z}_0}\left(q_i+\frac{1}{m}, r_i-\frac{1}{m} \right)$ for $m\in\N$.
    By construction, this is a disjoint union.
    Therefore, by \autoref{linesSegmentsZero}, for all $m\in\N$,
    \[
        \int_{E_m}\frac{1}{\sqrt{1+\|\vecu{u}\|^2}}\dif\mu\circ\varphi^{-1}\left(\vecu{u},v\right) =0.
    \]
    Consider $C\in\mathcal{C}$.
    Suppose $y_1\in\R$ is such that $(\vecz{z}_0,y_1)\in C$ and $(\vecz{z}_0,y_1)$ is not on a face of $C$ in $\mathcal{F}_1$.
    In particular, if $(\vecz{z}_0,y_1)$ is not on the interior of $C$, it lies only on a face of $C$ with normal vector orthogonal to $\eLast$.
    As $C$ has only finitely many faces, it follows there is $\delta>0$ such that $(\vecz{z}_0,y_1+\epsilon)\in C$ whenever $|\epsilon|<\delta$.
    Therefore, $(\vecz{z}_0,y_1)\in \text{ri}_{\vecz{z}_0}(C)$.
    \\
    Thus, $(y_0,\infty)\setminus J \subseteq \left\{y\in \R  \; | \; (\vecz{z}_0,y)\in \bigcup_{F\in \mathcal{F}_1}
        F \right\}.
    $ Further, for every $F\in \mathcal{F}_1$, $|F\cap (\{\vecz{z}_0\}\times(y_0,\infty))|\leq 1$.
    Therefore, as $\mathcal{F}_1$ is countable, $(y_0,\infty)\setminus J$ is countable.
    \\
    Suppose $b_0\in(y_0,\infty)\setminus J$.
    Then, $(\vecz{z}_0,b_0)\in\bigcup_{F\in \mathcal{F}_1} F\subseteq\bigcup_{\mathfrak{h}\in \mathcal{H}_1\cap \mathcal{H}_{supp}} \mathfrak{h}$.
    As $\vecz{z}_0\in S'$, by \autoref{zeroAboveSPrimeOnBoundaries},
    \begin{equation}
        \left\lvert\mu\circ\varphi^{-1}\right\rvert\left({\overbar{L}_{\vecz{z}_0}(b_0)}\right)=\left\lvert\mu\circ\varphi^{-1}\right\rvert\left(\left\{\left(\vecu{u},v\right)\in \Xi\times V\; | \; v=b_0+\vecu{u}\cdot \vecz{z}_0\right\}\right)=0.
        \label{pointsMeasureZero}
    \end{equation}
    Thus, as $(y_0,\infty)\setminus J$ is countable, $\left\lvert\mu\circ\varphi^{-1}\right\rvert\left(\bigcup_{b\in (y_0,\infty)\setminus J}{\overbar{L}_{\vecz{z}_0}(b)} \right)=0$.
    It follows for all $m\in\N$,
    \[
        \int_{{E_m}\cup\bigcup_{b\in(y_0,\infty)\setminus J}{\overbar{L}_{\vecz{z}_0}(b)}} \frac{1}{\sqrt{1+\|\vecu{u}\|^2}}\dif\mu\circ\varphi^{-1}\left(\vecu{u},v\right) =0.
    \]
    Further, $E_m\cup \bigcup_{b\in (y_0,\infty)\setminus J}{\overbar{L}_{\vecz{z}_0}(b)}\rightarrow\left\{\left(\vecu{u},v\right)\in \Xi\times V \; | \;v>y_0+\vecu{u}\cdot \vecz{z}_0\right\}$ as $m\rightarrow\infty$.
    By Dominated Convergence Theorem, as $\mu\circ\varphi^{-1}$ is finite,
    \[
        \int_{v>y_0+\vecu{u}\cdot \vecz{z}_0} \frac{1}{\sqrt{1+\|\vecu{u}\|^2}}\dif\mu\circ\varphi^{-1}\left(\vecu{u},v\right) =0.
    \]
    Similarly,
    \[
        \int_{v<y_0+\vecu{u}\cdot \vecz{z}_0} \frac{1}{\sqrt{1+\|\vecu{u}\|^2}} \dif\mu\circ\varphi^{-1}\left(\vecu{u},v\right)= 0.
    \]
    The equation $v= y_0 +\vecu{u}\cdot \vecz{z}_0$ is equivalent to $(\vecz{z}_0,-1)\cdot(\vecu{u},v)=-y_0$.
    Therefore, for all $\vecz{z}_0\in S'$ and $y_0\in \R$, when considering an open half-space $H$ in $\Xi\times V$ with a boundary defined by $(\vecz{z}_0,-1)\cdot(\vecu{u},v)=-y_0$, $\int_H \frac{1}{\sqrt{1+\|\vecu{u}\|^2}}\dif\mu\circ\varphi^{-1}\left(\vecu{u},v\right)=0$.
    \\
    \textbf{Conclusion with Radon-Nikodym}\\
    On the Borel sets of $\Xi\times V$, define the measure $\nu(E)=\int_{E} \frac{1}{\sqrt{1+\|\vecu{u}\|^2}} \dif\mu\circ\varphi^{-1}\left(\vecu{u},v\right)$.
    Given a Hahn decomposition of $\mu\circ\varphi^{-1}$ with positive set $P$ and negative set $N$, $\nu(E)=\int_{E} (\chi_P-\chi_N) \frac{1}{\sqrt{1+\|\vecu{u}\|^2}} \dif|\mu\circ\varphi^{-1}|\left(\vecu{u},v\right)$.
    \\
    By the previous part, if $H$ is an open half-space with normal vector $(\vecz{z}_0,-1)$ with $\vecz{z}_0\in S'$, then $\nu(H)=0$.\\
    Since $S'$ is dense in $\R^n$, $\nu(H)=0$ whenever the boundary of $H$ is in a dense set of directions.
    By a careful inspection of the proof of the Cramer-Wold theorem, it follows the characteristic function of $\nu$, $c_\nu(\vecx{t})\coloneqq \int_{\R^n} e^{i\vecx{t}\cdot\vecx{x}}\dif \nu(\vecx{x})$, is zero on a dense set of $\R^n$ \cite{cramerTheoremsDistributionFunctions1936a}*{Equation 4}.
    By Dominated Convergence Theorem, in fact $c_\nu\equiv 0$.
    Since characteristic functions are unique, $\nu$ is the zero measure \cite{klenkeCharacteristicFunctionsCentral2020}*{Theorem 15.9}.
    \\
    However, the Radon-Nikodym derivative of a measure is unique up to almost everywhere.
    As 0 is a Radon-Nikodym derivative for the zero measure and $(\chi_P-\chi_N)\frac{1}{\sqrt{1+\|\vecu{u}\|^2}}$ is never 0, it follows $|\mu\circ\varphi^{-1}|\left(\Xi\times V\right)=0$.
    \\
    Since $\varphi$ is bijective between $\gamma^{-1}[\mathcal{H}_1]$ and $\Xi\times V$, the support of $\mu$ is contained in
    \[
        \left(\mathcal{S}^n\times\R\right) \setminus \gamma^{-1}[\mathcal{H}_1]=\{ \veca{a}\in\mathcal{S}^n \; | \;  \veca{a}\cdot \eLast=0\}\times\R.
    \]
    By definition of $\mathcal{S}^n$, the support of $\mu$ is contained in a copy of $\mathcal{S}^{n-1}\times \R$ embedded into $\mathbb{S}^{n}\times \R$.
    That is, $f(\vecx{x})=\int_{\mathcal{S}^{n-1}\times\R} \sigma(\veca{a}\cdot\vecx{x}-b)-\sigma(-b)\dif\mu(\veca{a},b)$.
    Moreover, $g:\R^n\rightarrow \R$ defined as
    \begin{align*}
        g(\vecz{z}) & \coloneqq \int_{\mathcal{S}^{n-1}\times \R} \sigma\left(\bm{\alpha}\cdot \vecz{z}-b\right)-\sigma(-b)\dif\mu_{\mathcal{S}^{n-1}\times \R}\left(\bm{\alpha},b\right)               \\
                    & =\int_{\mathcal{S}^{n-1}\times \R} \sigma\left((\bm{\alpha},0)\cdot (\vecz{z},0)-b \right)-\sigma(-b)\dif\mu_{\mathcal{S}^{n-1}\times \R}\left(\bm{\alpha},b\right)=f(\vecz{z},0)
    \end{align*}
    is countably piecewise linear.
    By the inductive hypothesis, $\mu$ is the zero measure.
    \\
\end{proof}
To finish the proof of the main result [\ref{representableByDeltaPointMassesN}], it is necessary to split the measure into fully atomic and atomless parts and consider them separately.
By first establishing point masses always induce non-affineness even when dense, we can deduce the fully atomic and atomless components of the measure must both give rise to countably piecewise linear functions.
\begin{lemma}
    Let $\mu\in\mathcal{M}(\mathcal{S}^n\times \R)$ and $f(\vecx{x})\coloneqq \int_{\mathcal{S}^n\times\R} \sigma\left(\veca{a}\cdot\vecx{x} -b \right)-\sigma(-b)\dif\mu\left(\veca{a},b\right)$.
    Suppose $\mu\left(\left\{\left(\bm{c},d\right)\right\}\right)\not=0$.
    Then, $f(\vecx{x})$ is not affine on every open ball in the domain of $f$ intersecting $\hyp{c}{d}$.
    \label{PointMassesAreBoundaries}
\end{lemma}
\begin{proof}
    Rotate the coordinate system of $f$ such that $\bm{c}\cdot \eLast\not=0$.
    Define $\Xi\times V$, $\psi,$ and $\varphi$ as in \autoref{mainTheorem}.
    Let $S\subseteq\R^n$ be the set in \autoref{superUsefulSet}.
    \\
    There exists a unique function $\phi:\R^n\rightarrow \R$ such that $(\vecz{z},\phi(\vecz{z}))\in \hyp{c}{d}$ for all $\vecz{z}\in \R^n$.\\
    By way of contradiction, suppose $f$ is affine on an open ball $B_0$ intersecting $\hyp{c}{d}$.
    Then, there is an uncountable set $S'\subseteq S$ such that $\{\left(\vecz{z},\phi(\vecz{z})\right)\; | \; \vecz{z}\in S'\}\subseteq B_0$.
    \\
    For $\vecz{z}\in S'$, let $P_{\vecz{z}}= \left\{\left(\vecu{u},v\right)\in \Xi\times V\; | \; v = \phi(\vecz{z}) + \vecu{u}\cdot \vecz{z}\right\}$.
    Now, let $\mathcal{P}'_{\bm{c},d} = \{ P_{\vecz{z}}\setminus\left\{\psi(\hyp{c}{d})\right\} \; | \; \vecz{z}\in S'\}$.
    \\
    For unique $\vecz{z}_1,\ldots, \vecz{z}_{n+1}\in S'$, consider $\bigcap_{i\in [n+1]}
        P_{\vecz{z}_i}\setminus\left\{\psi(\hyp{c}{d})\right\}$.
    \\
    It follows any hyperplane in $ \psi^{-1}\left[ \bigcap_{i\in [n+1]}P_{\vecz{z}_i}\setminus\left\{\psi(\hyp{c}{d})\right\}\right]$ intersects the points $\{(\vecz{z}_1,\phi(\vecz{z}_1)),\ldots,(\vecz{z}_{n+1},\phi(\vecz{z}_{n+1}))\}$ in the domain where each $\vecz{z}_i\in S$.
    By \autoref{projectingUsefulSets}, these points do not lie on a common $n-1$ dimensional affine subspace, so $\hyp{c}{d}$ is the only hyperplane intersecting $\{(\vecz{z}_1,\phi(\vecz{z}_1)),\ldots,(\vecz{z}_{n+1},\phi(\vecz{z}_{n+1}))\}$.
    It follows $\bigcap_{i\in [n+1]}P_{\vecz{z}_i}\setminus\left\{\psi(\hyp{c}{d})\right\}=\emptyset$ and $|\mu\circ\varphi^{-1}|\left( \bigcap_{i\in [n+1]}P_{\vecz{z}_i}\setminus\left\{\psi(\hyp{c}{d})\right\} \right)=0$.
    \\
    Therefore, by \autoref{LimitTheIntersectionLimitTheWeighty},
    there are only countably many $P'_{\vecz{z}}\in \mathcal{P}'_{\bm{c},d}$ such that $|\mu\circ\varphi^{-1}|\left(P'_{\vecz{z}}\right)>0$.\\
    Since $S'$ is uncountable, there is $\vecz{z}_0\in S'$ such that $|\mu\circ\varphi^{-1}|\left(P_{\vecz{z}_0}\setminus\left\{\psi(\hyp{c}{d})\right\} \right)=0$.
    \\
    As $(\vecz{z}_0,\phi(\vecz{z}_0))\in B_0$, there is $\epsilon>0$ such that for all $\delta<\epsilon$, $f$ is affine on the line segment connecting $(\vecz{z}_0,\phi(\vecz{z}_0)-\delta)$ and $(\vecz{z}_0,\phi(\vecz{z}_0)+\delta)$.\\
    By \autoref{equalityOnLineImpliesZero} in \autoref{mainTheorem}, it follows $\int_{L_{\vecz{z}_0}(\phi(\vecz{z}_0)-\delta,\phi(\vecz{z}_0)+\delta)} \frac{1}{\sqrt{1+\|\vecu{u}\|^2}}\dif\mu\circ\varphi^{-1}\left(\vecu{u},v\right)=0$.
    Since this holds for all $\delta\in(0,\epsilon)$, by the Dominated Convergence Theorem,
    \begin{equation}
        \int_{{\overbar{L}_{\vecz{z}_0}(\phi(\vecz{z}_0))}} \frac{1}{\sqrt{1+\|\vecu{u}\|^2}}\dif\mu\circ\varphi^{-1}\left(\vecu{u},v\right)=0.
        \label{pointDerivIs0}
    \end{equation}
    Recall, $\overbar{L}_{\vecz{z}_0}(\phi(\vecz{z}_0))=\left\{\left(\vecu{u},v\right)\in \Xi\times V\; | \; v=\phi(\vecz{z}_0)+\vecu{u}\cdot \vecz{z}_0\right\}$.
    By \autoref{pointDerivIs0},
    \begin{align*}
        0 & =\int_{{\overbar{L}_{\vecz{z}_0}(\phi(\vecz{z}_0))}} \frac{1}{\sqrt{1+\|\vecu{u}\|^2}}\dif\mu\circ\varphi^{-1}\left(\vecu{u},v\right)                                                                                                                       \\
          & =\int_{P_{\vecz{z}_0}\setminus\left\{\psi(\hyp{c}{d})\right\}} \frac{1}{\sqrt{1+\|\vecu{u}\|^2}}\dif\mu\circ\varphi^{-1}\left(\vecu{u},v\right)+\int_{\{\psi(\hyp{c}{d})\}}\frac{1}{\sqrt{1+\|\vecu{u}\|^2}}\dif\mu\circ\varphi^{-1}\left(\vecu{u},v\right) \\
          & =0+\int_{\left\{ \left( \bm{c},d \right) \right\}} \veca{a}\cdot \eLast \dif\mu\left(\veca{a},b\right)=\mu\left(\left\{\left(\bm{c},d\right)\right\}\right)\cdot \left(\bm{c}\cdot\eLast\right).
    \end{align*}
    This is a contradiction.
\end{proof}
\begin{restatethis}{corollary}{countableRepresentablePointMasses}
    Let $f:\R^{n+1}\rightarrow\R$ be a continuous countably piecewise linear function.
    Suppose there is a countable collection $\mathcal{C}$ of convex polyhedra covering $\R^{n+1}$ such that $f$ is affine on each polyhedron and each polyhedron has non-empty interior.
    Suppose there exist \restateAlt{measuresMainResult}{$\mu\in\mathcal{M}(\mathbb{S}^n\times\R)$}{a finite, signed Borel measure $\mu$ on $\mathbb{S}^n\times\R$} and $c_0\in\R$ such that $f(\vecx{x})=\int_{\mathbb{S}^n\times\R}\sigma\left(\veca{a}\cdot\vecx{x}-b \right)-\sigma(-b)\dif\mu\left(\veca{a},b\right)+c_0$.\\
    Then, there are $r_0,r_{(\vecx{c},d)}\in\R$ and a countable set $M\subseteq \mathcal{S}^n\times\R$ such that $f(\vecx{x})=r_0+\sum_{(\veca{c},d)\in M}r_{(\veca{c},d)}\sigma(\veca{c}\cdot\vecx{x}-d)$.
  \end{restatethis}
  \begin{proof}
    By \autoref{packingDownIntoHalfsphere}, we can assume $f(\vecx{x})=\int_{\mathcal{S}^n\times \R} \sigma\left(\veca{a}\cdot\vecx{x}-b \right)-\sigma(-b)\dif\mu\left(\veca{a},b\right)+\veca{{a}}_0\cdot \vecx{x}+{b}_0$ for some $\mu\in\mathcal{M}(\mathcal{S}^n\times\R)$, $\veca{{a}}_0\in\R^{n+1}$, and ${b}_0\in\R$.
    Decompose $\mu$ such that $\mu=\mu_{C}+\sum_{(\veca{c},d)\in M}r_{(\veca{c},d)}\delta_{(\veca{c},d)}$ where $\mu_C$ is atomless, $M$ is a countable subset of $\mathcal{S}^n\times\R$, and $r_{(\veca{c},d)}\in\R\setminus\{0\}$ for all $(\veca{c},d)$.
    \\
    Let $g:\R^{n+1}\rightarrow \R$ be
    \begin{align*}
      g(\vecx{x}) & \coloneqq \int_{\mathcal{S}^n\times\R} \sigma\left(\veca{a}\cdot\vecx{x} -b \right)-\sigma(-b)\dif\left(\sum_{(\veca{c},d)\in M}r_{(\veca{c},d)}\delta_{(\veca{c},d)}\right)\left(\veca{a},b\right)+\veca{{a}}_0\cdot \vecx{x} +{b}_0 \\
                  & =\sum_{(\veca{c},d)\in M} r_{(\veca{c},d)}\left(\sigma\left(\veca{c}\cdot\vecx{x} -d \right)-\sigma(-d)\right) +\veca{{a}}_0\cdot \vecx{x} +{b}_0.
    \end{align*}
    Then, $g$ is certainly affine outside of $\bigcup_{(\veca{c},d)\in M}\hyp{c}{d}$.
    By \autoref{PointMassesAreBoundaries}, for every $C\in\mathcal{C}$ and every $(\veca{c},d)\in M$, $\left(\text{int } C\right)\cap \hyp{c}{d}=\varnothing$.
    Therefore, for every $C\in\mathcal{C}$, $g$ is affine on $C$, because $g$ is continuous and $\overline{\text{int } C}=C$.
    Thus, the cover $\mathcal{C}$ shows $g$ is countably piecewise linear.
    \\
    It follows $\int_{\mathcal{S}^n\times\R} \sigma\left(\veca{a}\cdot\vecx{x} -b \right)-\sigma(-b)\dif\mu_C(\veca{a},b)$ is also countably piecewise linear.
    By \autoref{mainTheorem}, $\mu_C$ is in fact the zero measure.\\
    Then, $f(\vecx{x})=g(\vecx{x})$.
    Note, $\veca{{a}}_0\cdot \vecx{x}=\sigma(\veca{{a}}_0\cdot\vecx{x})-\sigma(-\veca{{a}}_0\cdot\vecx{x})$.
    Thus,
    \[
      f(\vecx{x}) =\left(b_0-\sum_{(\bm{c},d)\in M}r_{(\bm{c},d)}\sigma(-d)\right) + \sigma(\veca{{a}}_0\cdot\vecx{x})-\sigma(-\veca{{a}}_0\cdot\vecx{x})+\sum_{(\bm{c},d)\in M} r_{(\bm{c},d)}\sigma(\bm{c}\cdot\vecx{x}-d).
    \]
  \end{proof}
  \begin{restatethis}{corollary}{representableByDeltaPointMassesN}
    Let $f:\R^{n+1}\rightarrow\R$ be a continuous finitely piecewise linear function.
    If there exist \restateAlt{measuresMainResult}{$\mu\in\mathcal{M}(\mathbb{S}^n\times\R)$}{a finite, signed Borel measure $\mu$ on $\mathbb{S}^n\times\R$} and $c_0\in\R$ such that $f(\vecx{x})=\int_{\mathbb{S}^n\times\R}\sigma\left(\veca{a}\cdot\vecx{x}-b \right)-\sigma(-b)\dif\mu\left(\veca{a},b\right)+c_0$, then $f$ is representable as a finite-width network as in \autoref{finiteNetwork}.
  \end{restatethis}
  \begin{proof}
    Since $f$ is \textit{finitely} piecewise linear, by \cite{gorokhovikGeometricalAnalyticalCharacteristic2011}, there exists a finite collection $\mathcal{C}$ of convex polyhedra with non-empty interior covering $\R^{n+1}$ such that $f$ is affine on each.
      By \autoref{countableRepresentablePointMasses}, there are $r_0\in\R$, $r_{(\vecx{c},d)}\in\R\setminus\{0\}$, and a countable set $M\subseteq \mathcal{S}^n\times\R$, such that $f(\vecx{x})=r_0+\sum_{(\veca{c},d)\in M}r_{(\veca{c},d)}\sigma(\veca{c}\cdot\vecx{x}-d)$.
      As $f$ will have a boundary at $\hyp{c}{d}$ for all $(\veca{c},d)\in M$, $M$ is a finite set.
  \end{proof}
\begin{corollary}
    Let $f(\vecx{x})=\int_{\mathbb{S}^n\times\R}\sigma\left(\veca{a}\cdot\vecx{x}-b \right)-\sigma(-b)\dif\mu\left(\veca{a},b\right)+c_0$ with $\mu\in\mathcal{M}(\mathbb{S}^n\times\R)$ and $c_0\in\R$.
    If $f\not\equiv 0$, then $f$ is not a compactly supported finitely piecewise linear function.
    \label{compactSupportImpossible}
\end{corollary}
\begin{proof}
    Suppose $f\not\equiv 0$.
    By \autoref{representableByDeltaPointMassesN}, $f$ is representable with a measure of the form $\sum_{(\veca{c},d)\in M}r_{(\veca{c},d)}\delta_{(\veca{c},d)}$ such that $M$ is finite.
    If $M$ is empty, $f(\vecx{x})=c_0\not=0$.
    Otherwise, $f$ will not be affine along $\hyp{c}{d}$ for some $(\veca{c},d)\in M$.
    Since $n\geq 2$, $\hyp{c}{d}$ will extend infinitely and $f$ cannot be compactly supported.
\end{proof}
\section*{Acknowledgements}
The author would like to acknowledge the support of the National Science Foundation (No.
1934884).
\bibliography{RidgeletTransformsTex}

\begin{bibdiv}
\begin{biblist}

\bib{aroraUnderstandingDeepNeural2018}{inproceedings}{
      author={Arora, Raman},
      author={Basu, Amitabh},
      author={Mianjy, Poorya},
      author={Mukherjee, Anirbit},
       title={Understanding {{Deep Neural Networks}} with {{Rectified Linear
  Units}}},
        date={2018-02},
   booktitle={International {{Conference}} on {{Learning Representations}}},
}

\bib{bachBreakingCurseDimensionality2017}{article}{
      author={Bach, Francis},
       title={Breaking the {{Curse}} of {{Dimensionality}} with {{Convex Neural
  Networks}}},
        date={2017},
     journal={Journal of Machine Learning Research},
      volume={18},
       pages={53},
}

\bib{bartlettValidGeneralizationSize}{inproceedings}{
      author={Bartlett, Peter~L},
       title={For {{Valid Generalization}} the {{Size}} of the {{Weights}} is
  {{More Important}} than the {{Size}} of the {{Network}}},
        date={1996},
   booktitle={Advances in neural information processing systems},
      volume={9},
   publisher={{MIT Press}},
       pages={7},
}

\bib{candesHarmonicAnalysisNeural1999}{article}{
      author={Cand{\`e}s, Emmanuel~J.},
       title={Harmonic {{Analysis}} of {{Neural Networks}}},
        date={1999-03},
        ISSN={1063-5203},
     journal={Applied and Computational Harmonic Analysis},
      volume={6},
      number={2},
       pages={197\ndash 218},
}

\bib{cramerTheoremsDistributionFunctions1936a}{article}{
      author={Cram{\'e}r, H.},
      author={Wold, H.},
       title={Some {{Theorems}} on {{Distribution Functions}}},
        date={1936},
        ISSN={1469-7750},
     journal={Journal of the London Mathematical Society},
      volume={s1-11},
      number={4},
       pages={290\ndash 294},
}

\bib{cybenkoApproximationSuperpositionsSigmoidal1989}{article}{
      author={Cybenko, G.},
       title={Approximation by superpositions of a sigmoidal function},
        date={1989-12},
        ISSN={1435-568X},
     journal={Mathematics of Control, Signals and Systems},
      volume={2},
      number={4},
       pages={303\ndash 314},
}

\bib{ePrioriEstimatesPopulation2019}{article}{
      author={E, Weinan},
      author={Ma, Chao},
      author={Wu, Lei},
       title={{\emph{A Priori}} estimates of the population risk for two-layer
  neural networks},
        date={2019},
        ISSN={1945-0796},
     journal={Communications in Mathematical Sciences},
      volume={17},
      number={5},
       pages={1407\ndash 1425},
}

\bib{e.RepresentationFormulasPointwise2022}{article}{
      author={E, Weinan},
      author={Wojtowytsch, Stephan},
       title={Representation formulas and pointwise properties for {{Barron}}
  functions},
        date={2022-04},
        ISSN={0944-2669, 1432-0835},
     journal={Calculus of Variations and Partial Differential Equations},
      volume={61},
      number={2},
       pages={46},
}

\bib{gorokhovikGeometricalAnalyticalCharacteristic2011}{misc}{
      author={Gorokhovik, V.~V.},
       title={Geometrical and analytical characteristic properties of piecewise
  affine mappings},
   publisher={{arXiv}},
        date={2011},
      number={arXiv:1111.1389},
}

\bib{halmosLargeIntersectionsLarge1992}{article}{
      author={Halmos, Paul~R.},
       title={Large {{Intersections}} of {{Large Sets}}},
        date={1992-04},
        ISSN={00029890},
     journal={The American Mathematical Monthly},
      volume={99},
      number={4},
       pages={307},
}

\bib{heReLUDeepNeural2020}{article}{
      author={He, Juncai},
      author={Li, Lin},
      author={Xu, Jinchao},
      author={Zheng, Chunyue},
       title={{{ReLU Deep Neural Networks}} and {{Linear Finite Elements}}},
        date={2020-06},
        ISSN={0254-9409, 1991-7139},
     journal={Journal of Computational Mathematics},
      volume={38},
      number={3},
       pages={502\ndash 527},
}

\bib{jacksonSclassGeneratedOpen1999}{article}{
      author={Jackson, Steve},
      author={Mauldin, R~Daniel},
       title={On the {$\sigma$}-class generated by open balls},
        date={1999},
     journal={Mathematical Proceedings of the Cambridge Philosophical Society},
      volume={127},
       pages={99\ndash 108},
}

\bib{klenkeCharacteristicFunctionsCentral2020}{incollection}{
      author={Klenke, Achim},
       title={Characteristic {{Functions}} and the {{Central Limit Theorem}}},
        date={2020},
   booktitle={Probability {{Theory}}: {{A Comprehensive Course}}},
      editor={Klenke, Achim},
      series={Universitext},
   publisher={{Springer International Publishing}},
     address={{Cham}},
       pages={327\ndash 366},
}

\bib{kostadinovaRidgeletTransformDistributions2014}{article}{
      author={Kostadinova, S.},
      author={Pilipovi{\'c}, S.},
      author={Saneva, K.},
      author={Vindas, J.},
       title={The {{Ridgelet}} transform of distributions},
        date={2014-05},
        ISSN={1065-2469},
     journal={Integral Transforms and Special Functions},
      volume={25},
      number={5},
       pages={344\ndash 358},
}

\bib{ongieFunctionSpaceView2019}{misc}{
      author={Ongie, Greg},
      author={Willett, Rebecca},
      author={Soudry, Daniel},
      author={Srebro, Nathan},
       title={A {{Function Space View}} of {{Bounded Norm Infinite Width ReLU
  Nets}}: {{The Multivariate Case}}},
   publisher={{arXiv}},
        date={2019},
      number={arXiv:1910.01635},
}

\bib{oostwalHiddenUnitSpecialization2021}{article}{
      author={Oostwal, Elisa},
      author={Straat, Michiel},
      author={Biehl, Michael},
       title={Hidden unit specialization in layered neural networks: {{ReLU}}
  vs. sigmoidal activation},
        date={2021-02},
        ISSN={0378-4371},
     journal={Physica A: Statistical Mechanics and its Applications},
      volume={564},
       pages={125517},
}

\bib{savareseHowInfiniteWidth2019}{misc}{
      author={Savarese, Pedro},
      author={Evron, Itay},
      author={Soudry, Daniel},
      author={Srebro, Nathan},
       title={How do infinite width bounded norm networks look in function
  space?},
   publisher={{arXiv}},
        date={2019},
      number={arXiv:1902.05040},
}

\bib{sonodaNeuralNetworkUnbounded2017}{article}{
      author={Sonoda, Sho},
      author={Murata, Noboru},
       title={Neural network with unbounded activation functions is universal
  approximator},
        date={2017-09},
        ISSN={10635203},
     journal={Applied and Computational Harmonic Analysis},
      volume={43},
      number={2},
       pages={233\ndash 268},
}

\bib{neumannSystemAlgebraischUnabhaengiger1928}{article}{
      author={v.~Neumann, J.},
       title={{Ein System algebraisch unabh\"angiger zahlen}},
        date={1928-12},
        ISSN={1432-1807},
     journal={Mathematische Annalen},
      volume={99},
      number={1},
       pages={134\ndash 141},
}

\end{biblist}
\end{bibdiv}
\end{document}